\documentclass[10pt,journal]{IEEEtran}
\usepackage{amsmath}           % Innehåller vissa matematiska
\usepackage{amssymb}
\usepackage{url}
\usepackage{cite}
\usepackage[latin1]{inputenc}  %Characters
\usepackage{graphicx}          % För att få grafik att fungera.

\usepackage{algorithm}
\usepackage{algpseudocode}
\usepackage{enumerate}
\usepackage{epstopdf}

\usepackage{amsthm}
\newtheorem{thm}{Theorem}
\theoremstyle{definition}
\theoremstyle{remark}
\newtheorem{rem}{Remark}%[section]
\theoremstyle{plain}

\DeclareMathOperator{\E}{E}

\DeclareMathOperator{\mvec}{vec}
\DeclareMathOperator{\tr}{tr}
\newcommand{\mbf}[1]{\mathbf{#1}}
\newcommand{\mbs}[1]{\boldsymbol{#1}}
\newcommand{\what}[1]{\widehat{#1}}
\newcommand{\wtilde}[1]{\widetilde{#1}}
\newcommand{\wbar}[1]{\overline{#1}}

\newcommand{\sign}{\text{sign}}

\newcommand{\eye}{\mbf{I}}
\newcommand{\covout}{\mbf{R}}
\newcommand{\out}{\mbf{y}}
\newcommand{\Out}{\mbf{Y}}
\newcommand{\inp}{\mbf{u}}
\newcommand{\data}{\mathcal{D}}
\newcommand{\linp}{\mbs{\Theta}}
\newcommand{\linpv}{\mbs{\theta}}
\newcommand{\condz}{\mbs{\zeta}}
\newcommand{\nonq}{\mbf{Z}}
\newcommand{\nonqv}{\mbf{z}}

\newcommand{\regpk}{\mbs{\varphi}}
\newcommand{\regp}{\mbs{\Phi}}
\newcommand{\regsquare}{\mbs{\Psi}}
\newcommand{\regpv}{\mbf{F}}
\newcommand{\regq}{\mbs{\Gamma}}
\newcommand{\gain}{\mbf{P}}
\newcommand{\regpregq}{\mbf{H}}
\newcommand{\regqk}{\mbs{\gamma}}
\newcommand{\regqv}{\mbf{G}}
\newcommand{\covq}{\mbf{D}}

\newcommand{\covcondz}{\mbf{\Sigma}_z}

\newcommand{\recerror}{\mbs{\xi}}
\newcommand{\error}{\mbs{\varepsilon}}

\newcommand{\covnoise}{\mbs{\Sigma}}
\newcommand{\0}{\mbf{0}}
\newcommand{\1}{\mbf{1}}
\newcommand{\weights}{\mbf{w}}
\newcommand{\matbase}{\mbf{E}}
\newcommand{\transp}{\top}
\newcommand{\recreg}{\mbf{c}}
\newcommand{\gammadiff}{\mbf{T}}

\newcommand{\cost}{V}
\newcommand{\para}{\mbs{\Omega}}
\def\NARMAX{\textsc{Narmax}}
\def\NARX{\textsc{Narx}}

\def\ARX{\textsc{Arx}}

\begin{document}

\title{Recursive nonlinear-system identification using latent variables} % Title, preferably not more 
\author{Per Mattsson, Dave Zachariah, Petre Stoica\thanks{ This work has been partly supported by the Swedish Research Council (VR) under contracts 621-2014-5874 and 2016-06079. Per Mattsson is with University of G\"avle. Dave Zachariah and Petre Stoica are with Uppsala University. }}
\maketitle

\begin{abstract}
In this paper we develop a method for learning nonlinear
system models with multiple outputs and inputs. We begin by modelling the
errors of a nominal predictor of the system using a latent variable
framework. Then using the maximum likelihood principle we derive a
criterion for learning the model. The resulting optimization problem
is tackled using a majorization-minimization
approach. Finally, we develop a convex majorization technique and show that it
enables a recursive identification method. The method learns parsimonious
predictive models and is tested on both synthetic and real nonlinear systems.
\end{abstract}
\section{Introduction}
In this paper we consider the problem of learning a nonlinear
dynamical system model with multiple outputs $\out(t)$ and multiple
inputs $\inp(t)$ (when they exist). Generally this identification problem can be tackled using
different model structures, with the class of linear models being arguably the
most well studied in engineering, statistics and econometrics
\cite{SoderstromStoica1988_system,Ljung1998_sysid,Bishop2006_pattern,Barber2012_bayesian,BoxEtAl2015_timeseries}. 

Linear models are often used even when the system is known to be
nonlinear \cite{Enqvist2005linear,SchoukensEtAl2016_linearnonlinear}. However certain nonlinearities,
such as saturations, cannot always be neglected. In such cases using block-oriented
models is a popular approach to capture static nonlinearities
\cite{GiriBai2010_block}. Recently, such models have been given
semiparametric formulations and identified using machine learning
methods, cf. \cite{Pillonetto2013_wiener,PillonettoEtAl2014_kernel}. To model nonlinear dynamics a common
approach is to use \NARMAX{} models \cite{Sjoberg1995_nonlinear,Billings2013_nonlinear}.

In this paper we are interested in recursive identification methods
\cite{Ljung1983_Theory}. In cases where the model structure is linear
in the parameters, recursive least-squares can be applied. For certain
models with
nonlinear parameters, the extended recursive least-squares has been
used \cite{Chen2004_extended}. Another popular approach is the recursive
prediction error method which has been developed, e.g., for Wiener models, Hammerstein models, and polynomial state-space models \cite{Wigren1993_wiener,Mattsson&Wigren2016_hammerstein,Tayamon2012_convergence}. 

Nonparametric models are often based on weighted sums of the
observed data \cite{Roll2005_nonlinear}. The weights vary for each predicted output and the number
of weights increases with each observed datapoint. The weights are
typically obtained in a batch manner;
in \cite{Bai2007_recursive, BijlEtAl2015_online} they are computed
recursively but must be recomputed for each new prediction of the
output.

For many nonlinear systems, however, linear models work well as an initial
approximation. The strategy in \cite{Paduart2010_nonlin_poly} exploits
this fact by first finding the best linear approximation using a frequency domain
approach. Then, starting from this approximation, a nonlinear
polynomial state-space model is fitted by solving a nonconvex
problem. This two-step method cannot be readily implemented
recursively and it requires input signals with appropriate frequency domain
properties.

In this paper, we start from a nominal model structure. This class can be based on insights about the system, e.g. that linear model structures can approximate a system around an operating point. Given a record of past outputs, $\out(t)$ and inputs $\inp(t)$, that is,
\begin{equation*}
\data_{t} \triangleq \bigl\{  \; \left(\out(1), \inp(1)  \right), \dots,  \left(\out(t) , \inp(t) \right) \; \bigr\},
\end{equation*}
a nominal model yields a predicted output $\out_0(t+1)$ which differs
from the output $\out(t+1)$. The resulting prediction error is denoted
$\error(t+1)$ \cite{Ljung1999_model}. By characterizing the nominal prediction errors in a data-driven manner, we aim to develop a refined predictor model of the system. Thus we integrate classic and data-driven system modeling approaches in a natural way.

The general model class and problem formulation are
introduced in Section~\ref{sec:modelstructure}. Then in
Section~\ref{sec:LAVAframework} we apply the principle of maximum
likelihood to derive a statistically motivated learning criterion. In
Section~\ref{sec:proposed} this nonconvex criterion is minimized using
a majorization-minimization approach that gives rise to
a convex user-parameter free method. We derive a computionally
efficient recursive algorithm for solving the convex problem,
which can be applied to large datasets as well as online learning scenarios. In Section~\ref{sec:numericalexamples}, we evaluate the proposed method using both
synthetic and real data examples.

In a nutshell, the contribution of the paper is a modelling approach and
identification method for nonlinear multiple input-multiple output systems that:
\begin{itemize}
\item explicitly separates modeling based on application-specific
  insights from general data-driven modelling,  
\item circumvents the choice of regularization parameters and
  initialization points,
\item learns parsimonious predictor models,
\item admits a computationally efficient implementation.
\end{itemize}

\emph{Notation:} $\matbase_{i,j}$ denotes the $ij$th standard
basis matrix. $\otimes$ and $\odot$ denote the Kronecker and Hadamard
products, respectively. $\text{vec}(\cdot)$ is the vectorization
operation. $\| \mbf{x} \|_2$, $\| \mbf{x} \|_1$ and $\| \mbf{X}
\|_{\mbf{W}} = \sqrt{ \text{tr}\{ \mbf{X}^\top \mbf{W} \mbf{X} \}}$,
where $\mbf{W} \succ \0$, denote $\ell_2$-, $\ell_1$- and weighted
norms, respectively. The Moore-Penrose pseudoinverse of $\mbf{X}$ is denoted $\mbf{X}^\dagger$.

\begin{rem} An implementation of the proposed method is available
at \url{https://github.com/magni84/lava}.
\end{rem}

\section{Problem formulation}\label{sec:modelstructure}

Given $\data_{t-1}$, the $n_y$-dimensional output of a system can always be written as
\begin{equation}\label{eq:out}
\out(t) = \out_0(t) + \error(t),
\end{equation}
where $\out_0(t)$ is any one-step-ahead predictor based on a nominal model. Here we consider nominal models on the form
\begin{equation}\label{eq:y0}
	\out_0(t) = \linp \regpk(t),
\end{equation}
where the $p \times 1$ vector $\regpk(t)$ is a given function of
$\data_{t-1}$ and $\linp$ denotes the unknown parameters.

\begin{rem}
A typical example of $\regpk(t)$ is
\begin{equation}\label{eq:linearphi}
	\regpk(t) = [\out^\top(t-1) \: \cdots \: \out^\top(t-n_a) \;
        \inp^\top(t-1) \: \cdots \: \inp^\top(t-n_b) \; 1]^\top,
\end{equation}
in which case the nominal predictor is linear in the data and
therefore captures the linear system dynamics. Nonlinearities can be
incorporated if such are known about the system, in which case
$\regpk(t)$ will be nonlinear in the data.
\end{rem}

The popular \ARX{} model structure, for instance, can be cast into
the framework \eqref{eq:out} and \eqref{eq:y0} by assuming that the
nominal prediction error $\error(t)$ is a white noise process
\cite{SoderstromStoica1988_system,Ljung1998_sysid}. For certain
systems, \eqref{eq:y0} may accurately describe the dynamics of the system around
its operation point and consequently the white noise assumption on $\error(t)$ may
be a reasonable approximation. However, this ceases to be the case
even for systems with weak nonlinearities, cf. \cite{Enqvist2005linear}.

Next, we develop a data-driven model of the prediction errors $\error(t)$ in \eqref{eq:out}, conditioned on the past data $\data_{t-1}$. Specifically, we assume the conditional model
\begin{equation}\label{eq:noise_model}
  \error(t) \: | \: \data_{t-1} \; \sim \; \mathcal{N}(\nonq \regqk(t),
\covnoise) ,
\end{equation}
where $\nonq$ is an $n_y \times q$ matrix of unknown latent
variables, $\covnoise$ is an unknown covariance matrix, and the $q \times 1$ vector $\regqk(t)$ is any given function of $\data_{t-1}$. 
This is a fairly general model structure that can capture 
correlated data-dependent nominal prediction errors.

Note that when $\nonq \equiv \0$, the prediction errors are temporally white and the nominal model \eqref{eq:y0} captures all relevant system dynamics. The latent variable is modeled as random here. Before data collection, we assume $\nonq$ to have mean $\0$ as we have no reason to depart from the nominal model assumptions until after observing data. Using a Gaussian distribution, we thus get
\begin{equation}\label{eq:pz}
\mvec(\nonq) \sim \mathcal{N}(\0, \covq),
\end{equation}
where $\covq$ is an unknown covariance matrix.

Our goal here is to identify a refined predictor of the form 
\begin{equation}\label{eq:predictor}
\boxed{\hat{\out}(t) = \underbrace{\what{\linp} \regpk(t)}_{\hat{\out}_0(t)} + \underbrace{\what{\nonq}\regqk(t)}_{\hat{\error}(t)},}
\end{equation}
from a data set $\data_{t-1}$, by maximizing the likelihood
function. The first term is an estimate of the nominal predictor model
while the second term tries to
capture structure in the data that is not taken into account by the nominal model.
Note that when $\what{\nonq}$ is sparse we obtain a
parsimonious predictor model.

%The identified model will thus be a nonlinear \ARX{} (\NARX) model. 

\begin{rem}
The model \eqref{eq:out}-\eqref{eq:noise_model} implies that we can write the output in the equivalent form
$$\out(t) = \linp \regpk(t) + \nonq \regqk(t) + \mbf{v}(t),$$
where $\mbf{v}(t)$ is a white process with covariance $\covnoise$. In
order to formulate a flexible data-driven error model
\eqref{eq:noise_model}, we overparametrize it using a high-dimensional $\regqk(t)$. In this case, regularization of $\nonq$ is desirable, which is achieved by \eqref{eq:pz}. Note that $\covq$ and $\covnoise$ are both unknown. Estimating these covariance matrices corresponds to using a data-adaptive regularization, as we show in subsequent sections.
\end{rem}

\begin{rem}
 The nonlinear function $\regqk(t)$ of $\data_{t-1}$ can
be seen as a basis expansion which is chosen to yield a flexible model
structure of the errors. 
In the examples below we will use the Laplace operator basis functions \cite{SolinSarkka2014_hilbert}. 
Other possible choices include the polynomial basis functions, Fourier
basis functions, wavelets, etc. \cite{Sjoberg1995_nonlinear,
  Ljung1998_sysid, vandenHofNinness2005_system}.
\end{rem}

\begin{rem} In \eqref{eq:predictor}, $\hat{\out}(t)$ is a
one-step-ahead predictor. However, the framework can be readily
applied to $k$-step-ahead prediction where $\regpk(t)$ and $\regqk(t)$
depend on $\out(1), \dots, \out(t-k)$.
\end{rem}

\section{Latent variable framework}\label{sec:LAVAframework}

Given a record of $N$ data samples, $\data_{N}$, our goal is to estimate $\linp$ and $\nonq$ to form the refined predictor \eqref{eq:predictor}. In Section \ref{sec:MML}, we employ the maximum likelihood approach based on the likelihood function $p(\Out |
\linp, \covq, \covnoise)$, which requires the estimation of nuisance parameters $\covq$ and $\covnoise$. For notational simplicity, we write the parameters as $\para = \{ \linp, \covq, \covnoise \}$ and in Section~\ref{sec:latent} we show how an estimator of $\nonq$ is obtained as a function of $\para$ and $\data_{N}$.

\subsection{Parameter estimation}\label{sec:MML}

We  write the output samples in matrix form as
\begin{equation*}
	\Out = \begin{bmatrix}\out(1) & \cdots & \out(N)\end{bmatrix} \in \mathbb{R}^{n_y \times N}.
\end{equation*}
In order to obtain maximum likelihood estimates of $\para$, we first derive the likelihood function by marginalizing out the latent variables from the data distribution:
\begin{equation}\label{eq:lik}
p(\Out | \para) = \int p(\Out | \para, \nonq) p(\nonq) d\nonq,
\end{equation}
where the data distribution $p(\Out | \para, \nonq)$ and $p(\nonq)$ are given by \eqref{eq:noise_model} and \eqref{eq:pz}, respectively.

To obtain a closed-form expression of \eqref{eq:lik}, we begin by constructing the regression matrices
\begin{align*}
	\regp &= \begin{bmatrix} \regpk(1) & \cdots & \regpk(N) \end{bmatrix} \in \mathbb{R}^{p \times N}, \\
	\regq &= \begin{bmatrix} \regqk(1) & \cdots & \regqk(N) \end{bmatrix} \in \mathbb{R}^{q \times N}.
\end{align*}
It is shown in Appendix~\ref{app:A} that \eqref{eq:lik} can be written as
\begin{equation}\label{eq:lik2}
	p(\Out | \para) = \frac{1}{(2\pi)^{N n_y} |\covout|} \exp\left( -\frac{1}{2} \| \out-\regpv \linpv \|^2_{\covout^{-1}} \right),
\end{equation}
where
\begin{equation}\label{eq:vec_variables}
\out = \mvec(\Out),  \quad \linpv = \mvec(\linp), \quad   \nonqv = \mvec(\nonq) ,
\end{equation}
are vectorized variables, and % we introduce the following matrices
\begin{align}
	\regpv &= \regp^{\top} \otimes \eye_{n_y}, \quad \regqv = \regq^{\top} \otimes \eye_{n_y},\label{eq:vec_reg} \\
	\covout &\triangleq \regqv \covq \regqv^{\top} + \eye_N \otimes \covnoise \label{eq:covout}.
\end{align}
 Note that \eqref{eq:lik2} is not a
Gaussian distribution in general, since $\covout$ may be a function of
$\Out$. It follows that maximizing \eqref{eq:lik2} is equivalent to solving
\begin{equation}\label{eq:MML}
	\min_{\para} \; \cost(\para),
\end{equation}
where
\begin{equation}
	\boxed{\cost(\para) \triangleq \| \out - \regpv \linpv \|^2_{\covout^{-1}}+ \ln | \covout |}
\label{eq:ML_func}
\end{equation}
and $\out - \regpv \linpv =  \mvec( \Out - \linp \regp  ) =
\mvec([\error(1) \: \cdots \: \error(N)])$ is nothing but the vector of
nominal prediction errors.

\subsection{Latent variable estimation}\label{sec:latent}

Next, we turn to the latent variable $\nonq$ which is used to model
the nominal prediction error $\error(t)$ in \eqref{eq:noise_model}. 
As we show in Appendix~\ref{app:A}, the conditional distribution of $\nonq$ is Gaussian and can be written as 
\begin{equation}\label{eq:condZ}
	p(\nonq | \para, \Out ) = \frac{1}{\sqrt{(2\pi)^{n_y q} |\covcondz|}} \exp\left( -\frac{1}{2} \| \nonqv - \condz \|_{\covcondz^{-1}}^2 \right),
\end{equation}
with conditional mean 
\begin{equation}\label{eq:condZmean}
	\condz = \covq \regqv^{\top} \covout^{-1} (\out - \regpv \linpv),
\end{equation}
and covariance matrix
\begin{equation}\label{eq:condZcov}
	\covcondz = \left(  \covq^{-1}  + \regqv^{\top} (\eye_N \otimes \covnoise^{-1}) \regqv \right)^{-1}.
\end{equation}
An estimate $\what{\nonq}$ is then given by evaluating the conditional (vectorized)
mean \eqref{eq:condZmean} at the
optimal estimate $\what{\para}$ obtained via \eqref{eq:MML}.

\section{Majorization-minimization approach}\label{sec:proposed}

The quantities in the refined predictor \eqref{eq:predictor} are
readily obtained from the solution of \eqref{eq:MML}. In general, $\cost(\para)$ may have local minima and \eqref{eq:MML} must be tackled using computationally efficient iterative methods to find the optimum. The obtained estimates will then depend on the choice of initial point $\para_0$.  Such methods includes the majorization-minimization approach \cite{HunterLange2004_mmtutorial,WuLange2010_mm}, which in turn include Expectation Maximization \cite{DempsterEtAl1977_em} as a special case.

The majorization-minimization approach is based on finding a majorizing function $V'(\para|\wtilde{\para})$ with the following properties:
\begin{equation}\label{eq:MMfunc}
\cost(\para) \leq \cost'(\para|\wtilde{\para}) , \quad \forall \para
\end{equation}
with equality when $\para = \wtilde{\para}$. The key is to find a majorizing function that is simpler to minimize than $\cost(\para)$. Let $\para_{k+1}$ denote the minimizer of $\cost'(\para|\para_{k})$. Then  
\begin{equation}\label{eq:MMsequence}
\begin{split}
\cost(\para_{k+1}) &\leq  \cost'(\para_{k+1}|\para_{k}) \leq
\cost'(\para_{k}|\para_{k}) = \cost(\para_{k}).
\end{split}
\end{equation}
This property leads directly to an iterative scheme that decreases $\cost(\para)$ monotonically, starting from an initial estimate $\para_0$.

Given the overparameterized error model \eqref{eq:noise_model}, it is natural to initialize at points in the parameter space which correspond to the nominal predictor model structure \eqref{eq:y0}. 
That is, 
\begin{equation}
\para_0 = \{ \linp_0, \0,\covnoise_0
\}, \qquad \text{where} \quad \covnoise_0 \succ \0, 
\label{eq:para_init}
\end{equation}
at which $\what{\nonq} \equiv \0$. 
%Below we derive a convex majorizing function $\cost'(\para|\wtilde{\para})$ of \eqref{eq:ML_func} which promotes parsimonious predictor models. 

% A well-known method of finding a majorizing function $V'(\para|\wtilde{\para})$ in latent variable problems is the Expectation Maximization method, which is derived in Appendix~\ref{app:EMalgorithm}. This method cannot, however, be initialized at points with strictly linear system dynamics but requires a more careful selection.

\subsection{Convex majorization}

For a parsimonious parameterization and computationally advantageous
formulation, we consider a diagonal structure of the covariance
matrices in \eqref{eq:noise_model}, i.e., we let
\begin{equation}\label{eq:covnoise}
\covnoise = \text{diag}(\sigma_1, \sigma_2, \dots, \sigma_{n_y}),
\end{equation}
and we let $\covq_i = \text{diag}( d_{i,1}, \ldots, d_{i,q} )$ denote the covariance matrix corresponding to the $i$th row of $\nonq$, so that 
\begin{equation}\label{eq:covq}
	\covq = \sum_{i=1}^{n_y} \covq_i \otimes \matbase_{i,i}.
\end{equation}
We begin by majorizing \eqref{eq:ML_func} via linearization of the concave term
$\ln |\covout|$:
\begin{equation}\label{eq:tangentplane}
\begin{split}
	\ln |\covout| &\leq \ln | \wtilde{\covout}| - \tr\{\wtilde{\covout}^{-1} (\eye_N \otimes \wtilde{\covnoise})\} - \tr\{\regqv^\transp \wtilde{\covout}^{-1}  \regqv \wtilde{\covq}\} \\
	&\quad + \tr\{\wtilde{\covout}^{-1} (\eye_N \otimes \covnoise)\} + \tr\{ \regqv^{\transp} \wtilde{\covout}^{-1} \regqv \covq \},
\end{split}
\end{equation}
where $\wtilde{\covq}$ and $\wtilde{\covnoise}$ are arbitrary diagonal
covariance matrices and $\wtilde{\covout}$ is obtained by inserting
$\wtilde{\covq}$ and $\wtilde{\covnoise}$ into \eqref{eq:covout}. 
The right-hand side of the inequality above is a majorizing tangent
plane to $\ln |\covout|$, see Appendix~\ref{app:lnapprox}. The use of \eqref{eq:tangentplane} yields a convex majorizing function of $V(\para)$ in \eqref{eq:MML}:
\begin{equation}\label{eq:majorizeV}
\begin{split}
V'(\para|\wtilde{\para}) &= \| \out - \regpv \linpv \|^2_{\covout^{-1}}+ \tr\{ \wtilde{\covout}^{-1} (\eye_N \otimes \covnoise) \} \\
&\quad  + \tr\{ \regqv^{\transp} \wtilde{\covout}^{-1} \regqv \covq \} +\wtilde{K},
\end{split}
\end{equation}
where $\wtilde{K} =  \ln | \wtilde{\covout}| - \tr\{
\wtilde{\covout}^{-1} (\eye_N \otimes \wtilde{\covnoise}) \} - \tr\{
\regqv^\transp \wtilde{\covout}^{-1}  \regqv \wtilde{\covq}\}$ is a
constant. To derive a computationally efficient algorithm for minimizing \eqref{eq:majorizeV}, the following theorem is useful:
\begin{thm}\label{thm:majorize}
The majorizing function \eqref{eq:majorizeV} can also be written as
\begin{equation}\label{eq:VprimZequality}
V'(\para|\wtilde{\para}) = \min_{\nonq} \;
V'(\para|\nonq, \wtilde{\para}) 
\end{equation}
where
\begin{equation}\label{eq:VprimZ}
\begin{split}
V'(\para|\nonq, \wtilde{\para}) &= \| \Out - \linp \regp - \nonq \regq
\|^2_{\covnoise^{-1}}+ \| \mvec(\nonq) \|_{\covq^{-1}}^2 \\ 
&\quad+ \tr\{ \wtilde{\covout}^{-1} (\eye_N \otimes \covnoise) \} + \tr\{ \regqv^{\transp} \wtilde{\covout}^{-1} \regqv \covq \} + \wtilde{K}  .
\end{split}
\end{equation}
The minimizing $\nonq$ is given by the conditional mean \eqref{eq:condZmean}.
\end{thm}

\begin{proof}
The problem in \eqref{eq:VprimZequality} has a minimizing $\nonq$
which, after vectorization, equals $\condz$ in \eqref{eq:condZmean}.
Inserting the minimizing $\nonq$ into \eqref{eq:VprimZ} and using \eqref{eq:the_equality} yields  \eqref{eq:majorizeV}.
\end{proof}

\begin{rem}
The augmented form in \eqref{eq:majorizeV}, enables us to solve for the nuisance parameters $\covq$ and $\covnoise$ in closed-form and also yields the conditional  mean $\what{\nonq}$ as a by-product.
\end{rem}

To prepare for the minimization of the function \eqref{eq:VprimZ} we write the matrix quantities using variables that denote the $i$th rows of the following matrices:
\begin{equation*}
\Out = \begin{bmatrix}
\vdots \\
\out^\top_i \\
\vdots
\end{bmatrix}, \quad 
\linp = \begin{bmatrix}
\vdots \\
\linpv^\top_i \\
\vdots
\end{bmatrix}, \quad 
\nonq = \begin{bmatrix}
\vdots \\
\nonqv^\top_i \\
\vdots
\end{bmatrix}, \quad
\regq = \begin{bmatrix}
\vdots \\
\regqk^\top_i \\
\vdots
\end{bmatrix}.
\end{equation*}

\begin{thm}\label{thm:l1_formulation}
	After concentrating out the nuisance parameters, the minimizing arguments $\linp$ and $\nonq$ of the function \eqref{eq:VprimZ} are obtained by solving the following convex problem:
\begin{equation}\label{eq:concentrated}
	\min_{\linp, \nonq}\; \sum_{i=1}^{n_y} \left( \|
            \out_i - \regp^{\top} \linpv_i - \regq^{\top} \nonqv_i
            \|_2 + \| \weights_i \odot \nonqv_i \|_1\right)
\end{equation}
where
\begin{align}
	\weights_i &=  \begin{bmatrix} w_{i, 1} & \cdots & w_{i, q} \end{bmatrix}^\top \label{eq:weights}, \quad w_{i,j} =  \sqrt{ \frac{{\regqk}_j^\transp \wtilde{\covout}_i^{-1}{\regqk}_j}{\tr\{\wtilde{\covout}_i^{-1}\}}}  \\
	\wtilde{\covout}_i &= \regq \wtilde{\covq}_i \regq^\transp + \wtilde{\sigma}_i \eye_N.
\end{align}
The closed-form expression for the minimizing nuisance parameters \eqref{eq:covnoise} and \eqref{eq:covq} are given by 
\begin{align}
	\hat{\sigma}_i &= \frac{\| \out_i - \regp^\transp \linpv_i - \regq^\transp \nonqv_i\|_2}{\sqrt{\tr\{\wtilde{\covout}_i^{-1}\}}},\,	\hat{d}_{i,j} =\frac{|z_{i,j}|}{\sqrt{ {\regqk}_j^\transp \wtilde{\covout}_i^{-1} {\regqk}_j}  }. \label{eq:dhat}
\end{align}
\end{thm}
\begin{proof}
	See Appendix \ref{app:l1_formulation}.
\end{proof}

%\begin{rem}
%By solving the convex problem \eqref{eq:concentrated}, we obtain the sought estimates of $\linp$ and $\nonq$ in \eqref{eq:predictor} directly.
%\end{rem}

\begin{rem}
Problem \eqref{eq:concentrated} contains a data-adaptive regularizing term which typically leads to parsimonious estimates of $\nonq$, cf. \cite{StoicaEtAl2014_weightedspice}.
\end{rem}
\begin{rem} Majorizing at a nominal predictor model,
i.e. $\wtilde{\para} = \{ \wtilde{\linp}, \0, \wtilde{\covnoise} \}$
as discussed above, yields $\wtilde{\covout}_i = \wtilde{\sigma}_i
\eye_N$ and the weights are given by 
\begin{equation}\label{eq:weights_omega0}
	w_{i,j} =\frac{\| {\regqk}_j \|_2}{\sqrt{N}}.
\end{equation}
Then problem \eqref{eq:concentrated} and consequently the minimization
of \eqref{eq:VprimZequality} becomes invariant with respect to
$\wtilde{\covnoise}$.
\end{rem}

The iterative scheme \eqref{eq:MMsequence} is executed by initializing
at  $\wtilde{\para} := \para_0$ and solving
\eqref{eq:concentrated}. The procedure is then repeated by updating
the majorization point using the new estimate $\wtilde{\para} := \what{\para}$.  It follows that the estimates will converge to a local minima of \eqref{eq:ML_func}.
The following theorem establishes that the local minima found, and thus the resulting predictor \eqref{eq:predictor}, is invariant to $\para_0$ in the form \eqref{eq:para_init}.

\begin{thm}\label{cor:omega_0} All initial points $\para_0$ in the form \eqref{eq:para_init} result in the same
sequence of minimizers $(\what{\linp}_k,\what{\nonq}_k)$ of \eqref{eq:concentrated}, for all $k>0$. Moreover, the sequence 
$(\what{\covq}_k,\what{\covnoise}_k)$ converges to a unique point when $k
\rightarrow \infty$.
\end{thm}

\begin{proof}
	See Appendix \ref{app:omega_0}.
\end{proof}

\begin{rem}
As a result we may initialize the scheme
\eqref{eq:MMsequence} at $\para_0 = \{ \0, \0, \eye_{n_y} \}$.
This obviates the need for carefully selecting an initialization point, which would be needed in e.g. the Expectation Maximization algorithm.
\end{rem}

%\begin{proof}
%The minimizing covariance parameter estimates are given in closed form as
%\begin{align*}
%	\hat{\sigma}_i &= \sqrt{\tilde{\sigma}_i}\| \out_i - \regp^{\top} \linpv_i - \regq^{\top} \nonqv_i \|_2/\sqrt{N} \\
%	\hat{d}_{i,j}  &=  \sqrt{\tilde{\sigma}_i} | z_{i,j} | / \| [
%        \regq^{\top} ]_j \|_2, \quad j=1, \dots, q
%\end{align*}
%for $i=1, \dots, n_y$. Inserting these expressions
%back in \eqref{eq:approxopt} yields:
%\[
%	\min_{\linp, \nonq, \covq, \covnoise}\, \sum_{i=1}^{n_y} \frac{2\sqrt{N}}{\sqrt{\tilde{\sigma}_i}}\left( \| \out_i - \regp^\transp \linpv_i - \regq^\transp \nonqv_i \|_2 + \| \weights \odot \nonqv_i \|_1 \right).
%\]
%Since each term of the sum can be minimized independently, the constant $2\sqrt{N}/\sqrt{\tilde{\sigma}}_i$ can be omitted and the proof is complete.
%\end{proof}

\subsection{Recursive computation}\label{sec:lava-r}
We now show that the convex problem \eqref{eq:concentrated} can be solved
recursively, for each new sample $\out(t)$ and $\inp(t)$.
\subsubsection{Computing $\what{\linp}$}
If we fix $\nonq$ and only solve for $\linp$, the solution is given by 
\begin{equation}\label{eq:est_linp}
	\what{\linp} = \wbar{\linp} - \nonq \regpregq^{\top}
\end{equation}
where 
\[
	\wbar{\linp} = \Out  \regp ^\dagger \quad \text{and} \quad \regpregq^{\top} =
        \regq \regp^\dagger.
\]
Note that both $\wbar{\linp}$ and $\regpregq$ are independent of
$\nonq$, and that they can be computed for each sample $t$ using a
standard recursive least-squares (LS) algorithm:
\begin{align}
	\wbar{\linp}(t) &= \wbar{\linp}(t-1) + (\out(t) - \wbar{\linp}(t-1) \regpk(t)) \regpk^\transp(t) \gain(t) \label{eq:RLS_linp}\\
	\regpregq(t) &= \regpregq(t-1) + \gain(t) \regpk(t) (
        \regqk^{\transp}(t) - \regpk^\transp(t)
        \regpregq(t-1)) \label{eq:RLS_H} \\
	\gain(t) &= \gain(t-1) - \frac{\gain(t-1) \regpk(t)\regpk^\transp(t) \gain(t-1)}{1 + \regpk^\transp(t) \gain(t-1) \regpk(t)}. \label{eq:RLS_gain} 
\end{align}

\begin{rem}
Natural initial values are $\wbar{\linp}(0) = \0$ and $\regpregq(0)
= \0$. The matrix $\regp^\dagger$ equals $\regp^\top (\regp \regp^\top)^{-1}$ when
$t\geq p$ samples yield a full-rank matrix $\regp$. The matrix
$\gain(t)$ converges to $(\regp \regp^\top)^{-1}$. A
common choice for the initial value of $\gain(t)$ is $\gain(0) = c \eye$, where a larger constant $c>0$ leads to a faster
convergence of \eqref{eq:RLS_gain},
cf. \cite{SoderstromStoica1988_system,StoicaAhgren2002_exactrls}.
\end{rem}

\subsubsection{Computing $\what{\nonq}$} Using \eqref{eq:est_linp}, we
concentrate out $\linp$ from \eqref{eq:concentrated} to obtain 
\[
	V'(\nonq) = \sum_{i=1}^{n_y} \left(  \| \recerror_i - \left( \regq^{\top} - \regp^{\top} \regpregq\right) \nonqv_i \|_2 + \| \weights_i \odot \nonqv_i \|_1 \right)
\]
where
\[
	\recerror_i = \out_i - \regp^{\top} \wbar{\linpv}_i.
\]
In Appendix \ref{app:algo} it is shown how the minimum of $V'(\nonq)$ can
be found via cyclic minimization with respect to the elements of
$\nonq$, similar to what has been done in
\cite{Zachariah&Stoica2015_onlinespice} in a simpler case. This iterative procedure is implemented using recursively
computed quantities and produces an estimate $\what{\nonq}(t)$ at sample $t$.

\subsubsection{Summary of the algorithm}\label{sec:algsummary} The algorithm
computes $\what{\linp}(t)$ and $\what{\nonq}(t)$ recursively by means of the following steps at each sample $t$:
\begin{enumerate}[i)]
\item Compute $\wbar{\linp}(t)$ and $\regpregq(t)$, using
  \eqref{eq:RLS_linp}-\eqref{eq:RLS_gain}.
\item Compute $\what{\nonq}(t)$ via the cyclic minimization of
  $V'(\nonq)$. Cycle through all elements $L$ times.
\item Compute $\what{\linp}(t) = \wbar{\linp}(t) - \what{\nonq}(t) \regpregq^{\top}(t)$
\end{enumerate}
The estimates are initialized as $\what{\linp}(0)=\0$ and
$\what{\nonq}(0)=\0$. In practice, small $L$ works well since we cycle
$L$ times through all elements of $\nonq$ for each new data
sample. The computational details are given in Algorithm~\ref{alg:rec}
in Appendix~\ref{app:algo}, which can be readily implemented e.g. in
\textsc{Matlab}. 

%The computational complexity for $N$ samples scales as
%$\mathcal{O}(n_y \text{max}(p^2,q^2 L) N )$. 

\section{Numerical experiments}\label{sec:numericalexamples}
In this section we evaluate the proposed method and compare it with
two alternative identification methods.

\subsection{Identification methods and experimental setup}\label{sec:setup}

The numerical experiments were conducted as follows. Three methods
have been used: LS identification of affine \ARX, \NARX{} using wavelet
networks (\textsc{Wave} for short), and the latent variable method (\textsc{Lava}) presented in this paper.
From
our numerical experiments we found that performing even only one iteration
of the majorization-minimization algorithm produces good results, and
doing so leads to a
computationally efficient recursive implementation (which we denote
\textsc{Lava-R} for \emph{r}ecursive).

For each method the function $\regpk(t)$ is taken as the linear regressor in \eqref{eq:linearphi}.
Then the dimension of $\regpk(t)$ equals $p = n_yn_a+ n_u n_b + 1$.
For affine ARX, the model is given by 
\[
	\hat{\out}_{ARX}(t) = \wbar{\linp} \regpk(t),
\]
where $\wbar{\linp}$ is estimated using recursive least squares \cite{SoderstromStoica1988_system}. 
Note that in \textsc{Lava-R}, $\wbar{\linp}$ is computed as a byproduct \eqref{eq:RLS_linp}.

For the wavelet network, \texttt{nlarx} in the System Identification
Toolbox for Matlab was used, with the number of nonlinear units automatically detected \cite{Ljung2007_Toolbox}.

For \textsc{Lava-R}, the model is given by \eqref{eq:predictor} and
$\what{\linp}, \what{\nonq}$ are found by the minimization of
\eqref{eq:concentrated} using $\wtilde{\para}_0 = \{ \0, \0,
\eye_{n_y} \}$. The minimization is performed using the recursive
algorithm in Section \ref{sec:algsummary} with $L=5$. The nonlinear function $\regqk(t)$ of the
data $\data_{t-1}$ can be chosen to be a set of
basis functions evaluated at $\regpk(t)$. Then $\nonq \regqk(t)$ can
be seen as a truncated basis expansion of some nonlinear function.
In the numerical examples, $\regqk(t)$ uses the Laplace basis
expansion due to its good approximation properties \cite{SolinSarkka2014_hilbert}. Each element in the expansion is given by 
\begin{equation}\label{eq:laplacebasis}
\gamma_{k_1, \dots, k_{p-1}}(t) = \prod^{p-1}_{i=1} \frac{1}{\sqrt{\ell_i}}\sin\left( \frac{
    \pi k_i(\varphi_i(t) + \ell_i)}{2\ell_i} \right),
\end{equation}
where $\ell_i$ are the boundaries of the inputs and outputs for each
channel and $k_i = 1,
\dots, M$ are the indices for each element of $\regqk(t)$. Then the
dimension of $$\regqk(t) = [\gamma_{1, \dots, 1}(t) \: \cdots
\: \gamma_{p-1, \dots, p-1}(t)]^\transp$$ equals $q = M^{p-1}$, where $M$
is a user parameter which determines the resolution of the basis.

Finally, an important part of the identification setup is the choice
of input signal. For a nonlinear system it is important to excite the
system both in frequency and in amplitude. For linear models a
commonly used input signal is a pseudorandom binary sequence (PRBS),
which is a signal that shifts between two levels in a certain
fashion. One reason for using PRBS is that it has good
correlation properties \cite{SoderstromStoica1988_system}. Hence,
PRBS excites the system well in frequency, but poorly in amplitude. A
remedy to the poor amplitude excitation is to multiply each interval
of constant signal level with a random factor that is uniformly distributed
on some interval $[0, A]$, cf. \cite{Wigren2006}. Hence, if the PRBS
takes the values $-1$ and $1$, then the resulting sequence will contain
constant intervals with random amplitudes between $-A$ and $A$. We denote
such a random amplitude sequence $\text{RS}(A)$ where $A$
is the maximum amplitude.

\subsection{Performance metric}
For the examples considered here the system does not necessarily
belong to the model class, and thus there is no true parameter vector
to compare with the estimated parameters. 
Hence, the different methods will instead be evaluated with respect to
the simulated model output $\hat{\out}_s(t)$. For \textsc{Lava-R}
\begin{align*}
	&\what{\regpk}(t) = [\hat{\out}^\top_s(t-1) \: \cdots \: \hat{\out}^\top_s(t-n_a) \;
        \inp^\top(t-1) \: \cdots \: \inp^\top(t-n_b) \; 1]^\top.\\
	&\what{\out}_s(t) = \what{\linp} \what{\regpk}(t) + \what{\nonq} \what{\regqk}(t)
\end{align*}
and $\what{\regqk}(t)$ is computed as
\eqref{eq:laplacebasis}  with $\regpk(t)$  replaced by $\hat{\regpk}(t)$.

The performance can then be evaluated using the root mean squared
error (RMSE) for each output channel $i$, 
\[
	\text{RMSE}_i = \sqrt{ \frac{1}{T} \sum_{t=1}^T\E \left[ \|
            y_i(t) - \hat{y}_{s,i}(t) \|_2^2\right ] }.
\]
The expectations are computed using 100 Monte Carlo simulations on validation data.

For the dataset collected from a real system, it is not possible to evaluate the expectation in the RMSE formula. For such sets we use the fit of the data, i.e.,
%a performance metric, i.e., 
\[
	\text{FIT}_i = 100\left( 1 - \frac{\| \out_i - \hat{\out}_{s,i} \|_2}{\|\out_i - \bar{y}_i \1  \|_2 }\right),
\]
where $\hat{\out}_{s,i}$ contains the simulated outputs for channel
$i$, $\bar{y}_i$ is the empirical mean of $\out_i$ and $\1$ is a
vector of ones. Hence, FIT compares the simulated output errors with those obtained using the empirical mean as the model output. 

\subsection{System with saturation}\label{ex:Saturation}
Consider the following state-space model, 
\begin{align}
	x_1(t+1) &= \text{sat}_2[0.9 x_1(t) + 0.1 u_1(t) ] \\ 
	x_2(t+1) &= 0.08 x_1(t) + 0.9 x_2(t) + 0.6 u_2(t) \\
	\out(t) &= \mbf{x}(t) + \mbf{e}(t).
\end{align}
where $\mbf{x}(t) = \begin{bmatrix} x_1(t) & x_2(t) \end{bmatrix}^\transp$ and 
\[
	\text{sat}_a(x) = \left\{ \begin{array}{ll} x & \text{if } |x| < a \\ \sign(x) a & \text{if } |x| \geq a \end{array}\right. .
\]
A block-diagram for the above system is shown in Fig.~\ref{fig:blockSaturation}.
The measurement noise $\mbf{e}(t)$ was chosen as a white Gaussian process with
covariance matrix $\sigma \eye$ where $\sigma = 2.5 \cdot 10^{-3}$. 

Data was collected from the system using an $\text{RS}(A)$ input signal
for several different amplitudes $A$. The identification was performed
using $n_a = 1$, $n_b = 1$, $M=4$, and $N = 1000$ data samples. This means that $p=5$ and $q=256$, and therefore there are 10 parameters in $\linp$ and $512$ in $\nonq$. 
\begin{figure}[h]
\centering
\includegraphics[width = 0.6\columnwidth]{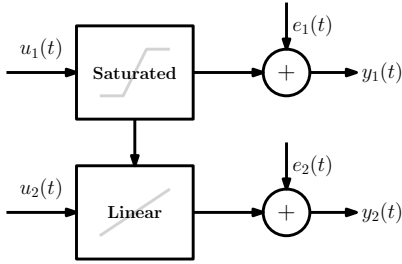}
\caption{A block diagram of the system used in Example \ref{ex:Saturation}.}
\label{fig:blockSaturation}
\end{figure}

Note that, for sufficiently low amplitudes $A$, $x_1(t)$ will be smaller than the
saturation level $a=2$ for
all $t$, and thus the system will behave as a linear system. However,
when $A$ increases, the saturation will affect the system output more and
more. The RMSE was computed for
eight different amplitudes $A$, and the result is shown in
Fig. \ref{fig:saturation}. It can be seen that for small amplitudes,
when the system is effectively linear, the \ARX{} model gives a
marginally better result than \textsc{Lava-R} and
\textsc{Wave}. However, as the amplitude is increased, the nonlinear
effects become more important, and \textsc{Lava-R} outperforms both \textsc{Wave} and \ARX{} models.

\begin{figure}[h]
\centering
\includegraphics[width = 0.8\linewidth]{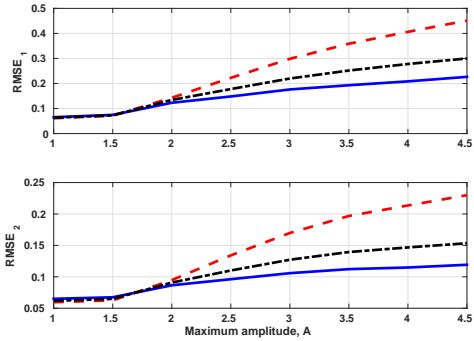}
\caption{The RMSE for Example~\eqref{ex:Saturation} computed for different input amplitudes, using \textsc{Lava-R} (solid), affine \ARX{} (dashed) and \textsc{Wave} (dash-dotted). }
\label{fig:saturation}
\end{figure}

\subsection{Water tank}\label{ex:WaterTank}
In this example a real cascade tank process is studied. It consists of
two tanks mounted on top of each other, with free outlets. The top
tank is fed with water by a pump. The input signal is the
voltage applied to the pump, and the output consists of the water
level in the two tanks. The setup is described in more detail in
\cite{Wigren2006}. The data set consists of 2500 samples collected
every five seconds. The first 1250 samples where used for
identification, and the last 1250 samples for validation. 

\begin{figure}[h]
\centering
\includegraphics[width = 0.8\linewidth]{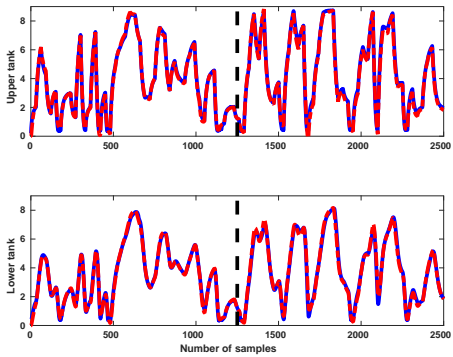}
\caption{The output in Example \ref{ex:WaterTank} (blue), plotted
  together with the output of the model identified by \textsc{Lava-R}
  (red). The system was identified using the first 1250 data
  samples. The validation set consisted of the remaining 1250
  samples.}
\label{fig:WaterTank}
\end{figure}

The identification was performed using $n_a = 2$, $n_b = 2$ and $M=3$. With two outputs, this results in 14 parameters in $\linp$ and 1458 parameters in $\nonq$. 
\textsc{Lava-R} found a model with only 37 nonzero parameters in $\nonq$, and the simulated output together with the measured output are shown in Fig. \ref{fig:WaterTank}. 
The FIT values, computed on the validation data are shown in Table
\ref{table:WaterTank}. It can be seen that an affine ARX model gives a 
good fit, but also that using \textsc{Lava-R} the FIT measure can be improved
significantly. In this
example, \textsc{Wave} did not perform very well.

\begin{table}
%	\begin{center}
	\caption{FIT for Example \ref{ex:WaterTank}.}
	\begin{tabular}{|c|c|c|c|}
		\hline
	                & \textsc{Lava-R} & \textsc{Wave} & \textsc{Arx} \\
		\hline
		Upper tank & $91.6\%$ & $79.2\%$ & $84.9\%$ \\
		Lower tank & $90.8\%$ & $76.9\%$  & $78.6\%$  \\
		\hline
	\end{tabular}
%	\end{center}
        \label{table:WaterTank}
\end{table}

\subsection{Pick-and-place machine}\label{ex:pick-place}
In the final example, a real pick-and-place machine is studied. This machine is used to place electronic components on a
circuit board, and is described in detail in
\cite{JuloskiEtAl2004_data}. 
%The machine can be in several different
%modes, with two major modes being the free mode and the impact mode. In
%the free mode, the machine is carrying an electronic component, but is
%not in contact with the circuit board. When the electronic component
%gets in contact with the circuit board the system switches to the impact
%mode. Besides these two modes, the system can also exhibit saturation
%and other nonlinearities.
This system exhibits saturation, different modes, and other nonlinearities. 
 The data used here are from a real physical process, and were also used in e.g. \cite{Bemporad2005,JuloskiEtAl2006_hybrid,OhlssonLjung2013_identification}.
The data set consists of a 15s recording of the single input $u(t)$ and the vertical position of the mounting head $y(t)$. The data was sampled at 50 Hz, and the first 8s ($N=400$) were used for identification and the last 7s for validation.

The identification was performed using $n_a = 2$, $n_b=2$ and
$M=6$. For the SISO system considered here, this results in 5
parameters in $\linp$ and 1296 parameters in
$\nonq$. \textsc{Lava-R} found a model with 33 of the parameters in
$\nonq$ being nonzero, the output of which is shown in Fig.~\ref{fig:pickandplace}. 
\begin{figure}[h]
\centering
\includegraphics[width = 0.8\linewidth]{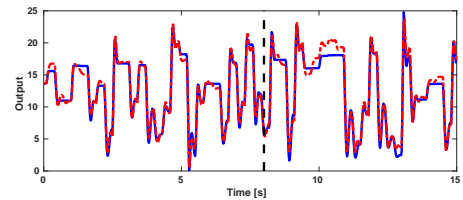}
\caption{The output in Example \ref{ex:pick-place} (blue), plotted
  together with the output of the model identified by \textsc{Lava-R}
  (red).  The system was identified using the first part of the data,
  while the validation set consisted of the remaining samples
  indicated after the dashed line.}
\label{fig:pickandplace}
\end{figure}

\begin{table}
	\caption{FIT for Example \ref{ex:pick-place}.}
	\begin{tabular}{|c|c|c|c|c|c|}
		\hline
	                & \textsc{Lava-R} & \textsc{Wave} & \textsc{Arx} \\
		\hline
		FIT & $83.2\%$ & $78.2\%$ & $73.1\%$ \\
		\hline
	\end{tabular}
        \label{table:PickAndPlace}
\end{table}
The FIT values, computed on the validation data, for \textsc{Lava-R}, \textsc{Wave} and affine \textsc{Arx} are shown in Table
\ref{table:PickAndPlace}. \textsc{Lava-R} outperforms \NARX{} using
wavelet networks, and both are better than \ARX{}.%, and also an improvement from results
%reported for piecewise ARX models reported in e.g. CITE.

\section{Conclusion}

We have developed a method for learning nonlinear systems with
multiple outputs and inputs. We began by modelling the errors of a
nominal predictor using a latent variable formulation. The nominal predictor could for instance be a linear approximation of the system but could also include known nonlinearities. A learning criterion was derived based on the principle of maximum likelihood, which obviates the tuning of regularization parameters. The criterion is then minimized using a
majorization-minimization approach. Specifically, we derived a convex
user-parameter free formulation, which led to a computationally
efficient recursive algorithm that can be applied to large
datasets as well as online learning problems.

The method introduced in this paper learns parsimonious predictor
models and captures nonlinear system dynamics. This was illustrated via
synthetic as well as real data examples. As shown in these examples a
recursive implementation of the proposed method was
capable of outperforming a batch method using a \NARX{} model with a
wavelet network.

\appendices
\section{Derivation of the distributions \eqref{eq:lik2} and \eqref{eq:condZ}}\label{app:A}
We start by computing $p(\Out | \para)$ given in \eqref{eq:lik}.  The function $p(\Out | \para, \nonq)$ can be found from \eqref{eq:out}--\eqref{eq:noise_model}  and the chain rule:
\begin{equation}\label{eq:y_given_z}
\begin{split}
	p(\Out | \para, \nonq) &=  \prod_{t=1}^N p_{\error}( \out(t) - \linp \regpk(t) | \data_{t-1}, \nonq),
\end{split}
\end{equation}
where we have neglected initial conditions \cite{SoderstromStoica1988_system}.
 Since
\begin{equation*}
\begin{split}
	&p_{\error}(\out(t) - \linp \regpk(t) | \data_{t-1}, \nonq)  \propto\\
      &  \exp\left( -\frac{1}{2}\| \out(t) - \linp \regpk(t) -
          \nonq \regqk(t) \|_{\covnoise^{-1}}^2 \right),
\end{split}
\end{equation*}
it follows that
\begin{equation}\label{eq:cond_lik}
\begin{split}
	& p(\Out | \para, \nonq) =\\
& \frac{1}{\sqrt{(2\pi)^{n_yN} | \covnoise|^{N} }} \exp\left( -\frac{1}{2}\| \Out - \linp \regp - \nonq
  \regq\|_{\covnoise^{-1}}^2 \right).
\end{split}
\end{equation}
Using the vectorized variable in \eqref{eq:vec_variables}-\eqref{eq:vec_reg} we can see that
\begin{equation*}
\mvec(\linp \regp) = \regpv \linpv \quad \text{and} \quad \mvec(\nonq \regq) = \regqv \nonqv.
\end{equation*}
and thus,
\[
	\| \Out - \linp \regp - \nonq \regq \|_{\covnoise^{-1}}^2 = \| \out - \regpv \linpv - \regqv \nonqv \|_{\eye_N \otimes \covnoise^{-1}}^2.
\]
Next, we note that the following useful equality holds:
\begin{multline}\label{eq:the_equality}
	\| \out - \regpv \linpv - \regqv \nonqv \|_{\eye_N \otimes \covnoise^{-1}}^2 + \| \nonqv \|_{\covq^{-1}}^2 = \\
	\| \out - \regpv \linpv \|_{\covout^{-1}}^2 + \| \nonqv - \condz \|_{\covcondz^{-1}}^2
\end{multline}
where $\covout$ is given by \eqref{eq:covout}, $\condz$ by \eqref{eq:condZmean}, and $\covcondz$ by \eqref{eq:condZcov}.
%\begin{align*}
% \covout &= \regqv \covq \regqv^{\top} + \eye_N \otimes \covnoise \\
% \covcondz &=\left( \covq^{-1} + \regqv^{\top} (\eye_N \otimes \covnoise^{-1}) \regqv \right)^{-1} \\
% \condz &= \covcondz \regqv^{\top}(\eye_N \otimes \covnoise^{-1}) (\out - \regpv \linpv) \\
% &= \covq \regqv^{\top} \covout^{-1} (\out - \regpv \linpv).
%\end{align*}
To see that the equality holds, expand the norms on both sides of \eqref{eq:the_equality} and apply the matrix inversion lemma.
%\begin{multline*}
%	(\out - \regpv \linpv)^{\top} \covout^{-1} (\out - \regpv \linpv) + \condz^{\top} \covcondz^{-1} \condz = \\
%(\out - \regpv \linpv)^{\top} (\eye_N \otimes \covnoise^{-1}) (\out - \regpv \linpv).
%\end{multline*}
%\subsection{Derivation of \eqref{eq:lik2} and \eqref{eq:condZ}}\label{app:lik2}

The sought-after distribution $p(\Out | \para)$ is given by
\eqref{eq:lik}. By using \eqref{eq:the_equality} it follows that 
\begin{multline}
	p(\Out | \para, \nonq) p(\nonq) \propto  \\
%	 \exp(-\frac{1}{2} \| \out - \regpv \linpv - \regqv \nonqv \|_{\eye_N \otimes \covnoise^{-1}}^2 ) \exp\left( -\frac{1}{2} \| \nonqv\|_{\covq^{-1}}^2 \right) \\
	 \exp(-\frac{1}{2} \| \out - \regpv \linpv\|_{\covout^{-1}}^2 ) \exp(-\frac{1}{2} \| \nonqv - \condz \|_{\covcondz^{-1}}^2) \label{eq:pY_pZ}.
\end{multline}
with the normalization constant $( (2\pi)^{n_y(N+q)} |\covnoise|^N|\covq|)^{-1/2}$.
%\[
%	\frac{1}{\sqrt{(2\pi)^{n_y(N+q)}} | \covnoise|^{N} | \covq|}}.
%\]
Noting that
\begin{equation*}
	\int \exp\left(- \frac{1}{2} \| \nonqv - \condz \|_{\covcondz^{-1}}^2 \right) d\nonq = \sqrt{ (2\pi)^{n_y q} | \covcondz | }
\end{equation*}
%it follows that
%\begin{multline}\label{eq:py_theta}
%	p(\Out | \para) = \frac{1}{\sqrt{ (2\pi)^{N n_y} | \covnoise|^N |\covq| |\covcondz^{-1}|}} \times \\ \exp\left( -\frac{1}{2} \| \out - \regpv \linpv \|_{\covout^{-1}}^2 \right).
%\end{multline}
%Because
%\begin{multline*}
%	| \covnoise |^{N} |\covq| | \covcondz^{-1}| = | \eye_N \otimes \covnoise | | \covq \covcondz^{-1}| = \\
%	| \eye_N \otimes \covnoise | | \eye_{q n_y} + \covq \regqv^{\top}(\eye_N \otimes \covnoise^{-1}) \regqv | =  \\
%	| \eye_N \otimes \covnoise | | \eye_{N n_y} + \regqv \covq \regqv^{\top} (\eye_N \otimes \covnoise^{-1}) | = | \covout|
%\end{multline*}
it can be seen that
\begin{equation}\label{eq:py_theta}
	p(\Out | \para) = \frac{1}{\sqrt{(2\pi)^{N n_y} |\covout|}} \exp\left( -\frac{1}{2} \| \out - \regpv \linpv \|_{\covout^{-1}}^2 \right),
\end{equation}
which proves \eqref{eq:lik2}. To obtain an expression for $p(\nonq | \para, \Out)$ simply insert \eqref{eq:pY_pZ} and \eqref{eq:py_theta} into Bayes' rule
%\[
%	p(\nonq | \para, \Out) = \frac{p(\Out | \para, \nonq) p(\nonq)}{p(\Out|\para)},
%\]
to get \eqref{eq:condZ}.

\section{Derivation of the majorizing tangent plane \eqref{eq:tangentplane}}\label{app:lnapprox}
The first-order Taylor expansion of the log-determinant can be written as
\begin{equation*}
\begin{split}
	\ln | \covout| &\simeq \ln | \wtilde{\covout}|
	+ (\partial_{\mbs{\sigma}} \ln|\covout|)|_{\covout =
	  \wtilde{\covout}} (\mbs{\sigma}-\wtilde{\mbs{\sigma}}) \\
	&\quad +
	(\partial_{\mbf{d}} \ln |\covout|_{\covout =
	  \wtilde{\covout}})(\mbf{d} - \tilde{\mbf{d}})
\end{split}
\end{equation*}
where $\mbs{\sigma}$ is the vector of diagonal elements in $\covnoise$ and $\mbf{d}$ contains the diagonal elements in $\covq$.

For the derivatives with respect to $\mbf{d}$ we have
\[
	\frac{\partial}{ \partial d_{i,j}} \ln |\covout| = \tr\left( \covout^{-1} \frac{\partial \covout}{\partial d_{i,j}} \right) = \tr\left( \regqv^\transp \covout^{-1} \regqv \frac{\partial \covq}{\partial{d_{i,j}}}  \right).
\]
 Note that
 \[
	 \sum_{i=1}^{n_y} \sum_{j=1}^q d_{i,j}\frac{\partial \covq}{\partial d_{i,j}} = \covq
 \]
 so 
Hence
\[
	\partial_{\mbf{d}} \ln |\covout|_{\covout = \wtilde{\covout}}  (\mbf{d} -\tilde{\mbf{d}}) =  \tr\left( \regqv^\transp \wtilde{\covout}^{-1} \regqv (\covq - \wtilde{\covq})  \right).
\]
In the same way 
\[
	\partial_{\mbs{\sigma}} \ln |\covout|_{\covout = \wtilde{\covout}} (\mbs{\sigma}-\tilde{\mbs{\sigma}}) = \tr\left( \wtilde{\covout}^{-1} (\eye_N \otimes (\covnoise -\wtilde{\covnoise}) \right).
\]
Since $\ln| \covout|$ is concave in $\mbs{\sigma}$ and $\mbf{d}$, it follows that 
\begin{equation}\label{eq:logdet_ineq}
	\ln | \covout| \leq \wtilde{K} + \tr\left( \regqv^\transp \wtilde{\covout}^{-1} \regqv \covq  \right) + \tr\left( \wtilde{\covout}^{-1} (\eye_N \otimes \covnoise) \right)
\end{equation}
where
\[
	\wtilde{K} = \ln |\wtilde{\covout}| - \tr\left( \regqv^\transp \wtilde{\covout}^{-1} \regqv \wtilde{\covq} \right) - \tr\left( \wtilde{\covout}^{-1} (\eye_N \otimes \wtilde{\covnoise}) \right).
\]

\section{Proof of Theorem \ref{thm:l1_formulation}} \label{app:l1_formulation}
It follows from \eqref{eq:covq} that
\[
	\covout = \sum_{i=1}^{n_y} ( \covout_i \otimes \matbase_{i,i})
\]
%\begin{align*}
%	\covout &= \sum_{i=1}^{n_y} \left( \regqv^\transp (\covq_i \otimes \matbase_{i,i} ) \regqv + \sigma_i (\eye_N \otimes \matbase_{i,i}) \right) \\
%	&= \sum_{i=1}^{n_y} \left( \covout_i \otimes \matbase_{i,i} \right),
%\end{align*}
where $\covout_i = \regq^\transp \covq_i \regq + \sigma_i \eye_N$. Hence,
\[
	\covout^{-1} = \sum_{i=1}^{n_y} \left( \covout_i^{-1} \otimes \matbase_{i,i} \right).
\]
Thus, we can rewrite \eqref{eq:VprimZ} as (to within an additive constant):
\begin{multline}\label{eq:Vprim_sum}
	\cost'(\para | \nonq, \wtilde{\para}) = \sum_{i=1}^{n_y}\Bigl( \frac{1}{\sigma_i} \| \bar{\out}_i \|_2^2 + \| \nonqv_i \|_{\covq_i^{-1}}^2 + \\ \sigma_i \tr(\wtilde{\covout}_i^{-1}) + \tr( \regq \wtilde{\covout}_i^{-1} \regq^\transp \covq_i ) \Bigr).
\end{multline}
where $\bar{\out}_i = \out_i - \regp^\transp \linpv_i - \regq^\transp \nonqv_i$.

We next derive analytical expressions for the $\covnoise$ and $\covq$ that minimize $\cost'(\para | \nonq, \wtilde{\para})$.  Note that
\[
	\frac{\partial}{\partial \sigma_i} \cost'(\para | \nonq, \wtilde{\para}) = -\frac{1}{\sigma_i^2} \| \bar{\out}_i \|_2^2 + \tr(\wtilde{\covout}_i^{-1}),  
\]
and setting the derivative to zero yields the estimate \eqref{eq:dhat}.
%\[
%	\hat{\sigma}_i = \frac{\| \bar{\out}_i \|_2}{\sqrt{\tr(\wtilde{\covout}_i^{-1})}}.
%\]
In the same way it can be seen that the minimum of $d_{i,j}$ is attained at \eqref{eq:dhat}.
%Taking the derivative with respect to $d_{i,j}$,
%\begin{align*}
%	\frac{\partial}{\partial d_{i,j}} \cost'(\para | \nonq, \para_n) &= -\frac{1}{d_{i,j}^2} z_{i,j}^2 + \tr( \regq \wtilde{\covout}^{-1}_i \regq^{\transp} \matbase_{j,j}) \\
%	&= -\frac{1}{d_{i,j}^2} z_{i,j}^2 + \regqk_j^\transp \wtilde{\covout}_i^{-1} \regqk_j.
%\end{align*}
%and noting that 
%\[
%	\tr( \regq \wtilde{\covout}^{-1}_i \regq^{\transp} \matbase_{j,j}) = \wbar{\regqk}^\transp_j \wtilde{\covout}^{-1}_i \wbar{\regqk}_j
%\]
%we get the estimate 
%\[
%	\hat{d}_{i,j} = \frac{|z_{i,j}|}{\sqrt{ \wbar{\regqk}_j^\transp \wtilde{\covout}_i^{-1} \wbar{\regqk}_j}}.
%\]

Inserting $\hat{\sigma}_i$ and $\hat{d}_{i,j}$ into \eqref{eq:Vprim_sum}, we see that we can find the minimizing $\linp$ and $\nonq$ by minimizing
\begin{align*}
	2	&\sum_{i=1}^{n_y} \Bigl(\sqrt{\tr(\wtilde{\covout}_i^{-1})} \| \out_i - \regp^\transp \linpv_i - \regq^\transp \nonqv_i \|_2 + \\
	&\sum_{j=1}^q |z_{i,j}| \sqrt{ \wbar{\regqk}_j^\transp \wtilde{\covout}_i^{-1} \wbar{\regqk}_j } \Bigr).
\end{align*}
Since term $i$ in the above sum is invariant with respect to $\linpv_k$ and $\nonqv_k$ for $k\neq i$, we can divide term $i$ by $2\sqrt{\tr(\wtilde{\covout}_i^{-1})}$, and see that minimizing the criterion above is equivalent to \eqref{eq:concentrated}. 
%\[
%	\sum_{i=1}^{n_y} \left(\| \out_i - \regp^\transp \linpv_i - \regq^\transp \nonqv_i \|_2 + \| \weights_i \odot \nonqv_i \|_1 \right).
%\]
\section{Proof of Theorem \ref{cor:omega_0}}\label{app:omega_0}

Initializing \eqref{eq:MMsequence} at
$\wbar{\para}_0 = \{ \0, \0, \eye_{n_y}\}$ and $\para_0 = \{ \linp_0,
\0, \covnoise_0\}$ where $\covnoise_0 = \text{diag}(
\sigma^{(0)}_1, \ldots, \sigma^{(0)}_{n_y} )$, produces two
sequences denoted $\wbar{\para}_k = \{ \wbar{\linp}_k, \wbar{\covq}_k,
\wbar{\covnoise}_k\}$ and $\para_k = \{ \linp_k,
\covq_k,\covnoise_k\}$ for $k>0$, respectively. This results also in
sequences $\wbar{\nonq}_k$ and $\nonq_k$.
The theorem states that:
\begin{equation}\label{eq:convergepara}
	\linp_k = \wbar{\linp}_k  \quad \text{and} \quad  \nonq_k = \wbar{\nonq}_k
\end{equation}
\begin{equation}\label{eq:convergecov}
	\covq_k - \wbar{\covq}_k \rightarrow \0  \quad \text{and} \quad \covnoise_k -
        \wbar{\covnoise}_k \rightarrow \0.
\end{equation}
We now show the stronger result that the covariance matrices converge as
\begin{equation}\label{eq:induction}
	\covq_{i}^{(k)} = c_i^{(k)}
        \wbar{\covq}_{i}^{(k)}, \quad \sigma_i^{(k)} = c_i^{(k)}
        \wbar{\sigma}_i^{(k)}, \quad \forall k >0,
\end{equation}
where $c_i^{(k)} = (\sigma^{(0)}_i)^{\frac{1}{2^k}}$. Note that $c_i^{(k)}
\rightarrow 1$ as $k \rightarrow \infty$. Hence \eqref{eq:induction} implies
\eqref{eq:convergecov}. We prove
\eqref{eq:induction} and \eqref{eq:convergepara} by induction.
That \eqref{eq:induction} and \eqref{eq:convergepara} holds for $k = 1$ follows directly from Theorem \ref{thm:l1_formulation}. Now assume that \eqref{eq:induction} holds for some $k\geq 1$.  Let 
\begin{align*}
	\wbar{\covout}_i &= \regq \wbar{\covq}_{i}^{(k)} \regq^\transp + \wbar{\sigma}_i^{(k)} \eye_N,\\
	\wtilde{\covout}_i &= \regq \covq_i^{(k)} \regq^\transp + \sigma_i^{(k)} \eye_N = c_i^{(k)} \wbar{\covout}_i,
\end{align*}
where the last equality follows by the assumption in
\eqref{eq:induction}. Therefore the weights used to estimate
$\linp_{k+1}$ and $\nonq_{k+1}$ are the same as those used to estimate $\wbar{\linp}_{k+1},\wbar{\nonq}_{k+1}$:
\[
	w_{i,j} = \sqrt{\frac{{\regqk}_j^\transp \wtilde{\covout}_i^{-1} {\regqk}_j}{\tr(\wtilde{\covout}_i^{-1})}} =\sqrt{\frac{{\regqk}_j^\transp \wbar{\covout}_i^{-1} {\regqk}_j}{\tr(\wbar{\covout}_i^{-1})}} ,
\]
so we can conclude that
$\linp_{k+1} = \wbar{\linp}_{k+1}$ and $\nonq_{k+1} = \wbar{\nonq}_{k+1}$. The estimate $\covq_{k+1}$ is given by 
\[
	d_{i,j}^{(k+1)} = \frac{|z_{i,j}|}{\sqrt{{\regqk}_j^\transp \wtilde{\covout}_i^{-1} {\regqk}_j}} = \frac{\sqrt{c_i^{(k)}} |z_{i,j}|}{\sqrt{{\regqk}^\transp \wbar{\covout}^{-1} {\regqk}}} = c_i^{(k+1)} \wbar{d}_{i,j}^{(k+1)}
\]
so $\covq_{i}^{(k+1)} = c_i^{(k+1)} \wbar{\covq}_i^{(k+1)}$, and in
the same way it can be seen that $\sigma_i^{(k+1)} = c_i^{(k+1)}
\wbar{\sigma}_i^{(k+1)}$. Hence by induction \eqref{eq:induction} and \eqref{eq:convergepara}  are
true for all $k>0$ and Theorem~\ref{cor:omega_0} follows.

\section{Derivation of the proposed recursive algorithm}\label{app:algo}
In order to minimize $\cost'(\nonq)$ we use a cyclic algorithm. That is, we minimize with respect to one component at a time. We follow an approach similar to that in \cite{Zachariah&Stoica2015_onlinespice}, with the main difference being that here we consider arbitrary nonnegative weights $\weights_i$.

Note that minimization of $\cost'(\nonq)$ with respect to $z_{i,j}$ is equivalent to minimizing 
\[
	\cost'(z_{i,j}) = \| \tilde{\recerror}_{i,j} - \recreg_j z_{i,j} \|_2 + w_{i,j} |z_{i,j}|
\]
where
\[
	\tilde{\recerror}_{i,j} = \recerror_i - \sum_{k\neq j} \recreg_kz_{i,k}, \quad \recreg_j = [\regq^{\top} - \regp^\top \regpregq]_j.
\]
As in \cite{Zachariah&Stoica2015_onlinespice} it can be shown that the sign of the optimal $\hat{z}_{i,j}$ is given by
	$\sign(\hat{z}_{i,j}) = \sign( \recreg_j^{\transp} \tilde{\recerror}_{i,j}).$
Hence we only have to find the absolute value $r_{i,j} =
|z_{i,j}|$. Let
\[
	\alpha_{i,j} = \| \tilde{\recerror}_{i,j} \|_2^2 ,\quad \beta_j = \| \recreg_j \|_2^2 , \quad	g_{i,j} = \recreg_j^\transp \tilde{\recerror}_{i,j}.
\]
It is then straightforward to verify that the minimization of $\cost'(z_{i,j})$ is equivalent to minimizing 
\[
	\cost'(r_{i,j}) = (\alpha_{i,j} + \beta_j r_{i,j}^2 - 2 g_{i,j} r_{i,j})^{1/2} + w_{i,j} r_{i,j},
\]
over all $r_{i,j} \geq 0$, and then setting $\hat{z}_{i,j} = \sign(g_{i,j}) \hat{r}_{i,j}$. From the Cauchy-Schwartz inequality it follows that
\[
	\alpha_{i,j} \beta_{i,j} \geq g_{i,j}^2.
\]
Using this inequality it was shown in \cite{Zachariah&Stoica2015_onlinespice} that $\cost'(r_{i,j})$ is a convex function. The derivative of $\cost'(r_{i,j})$ is given by (dropping the subindices for now), 
\begin{equation}\label{eq:dJdr}
	\frac{d\cost'}{dr} = \frac{\beta r - |g|}{(\beta r^2 - 2|g|r + \alpha)^{1/2}} + w.
\end{equation}
Since $\cost'(r)$ is convex it follows that it is minimized by $r=0$ if and only if $d\cost'(0)/dr \geq 0$, i.e., if and only if 
\begin{equation}\label{eq:zero_ineq}
	\alpha w^2 \geq g^2.
\end{equation}
Next we study the case when $g^2 > \alpha w^2$.
It then follows from \eqref{eq:dJdr} that the stationary points of $\cost'(r)$ satisfy 
\begin{equation}\label{eq:stationary1}
	(\beta r - |g| ) = -w (\beta r^2 - 2 |g| r + \alpha)^{1/2}
\end{equation}
Solving this equation for $r$ we get the stationary point 
\[
	\hat{r} = \frac{|g|}{\beta} - \frac{w}{\beta\sqrt{\beta - w^2}} \sqrt{\alpha \beta - g^2}.
\]
Hence we can conclude that the minimizer of $\cost'(z_{i,j})$ is given by
\[
	\hat{z}_{i,j} = \left\{\begin{array}{ll} \sign(g_{i,j}) \hat{r}_{i,j} & \text{if } \alpha_{i,j} w_{i,j}^2 < g_{i,j}^2 \\ 0 & \text{otherwise} \end{array} \right. .
\]
Next we show how to obtain this estimate using only recursively computed quantities. 
Let
\begin{align}
	\gammadiff &= ( \regq^{\top} - \regp^{\top}\regpregq)^{\top} ( \regq^{\top} - \regp^{\top} \regpregq) \label{eq:gammadiff} \\
	\kappa_i &= \| \recerror_i \|_2^2 \\
	\rho_i &= (\regq^\top - \regp^\top \regpregq)^\top \recerror_i \\
	\eta_i &= \| \recerror_i - (\regq^{\top} - \regp^{\top} \regpregq) \nonqv_i \|_2^2 \\
	\zeta_i &= (\regq^{\top} - \regp^{\top} \regpregq) (\recerror_i - (\regq^{\top} - \regp^{\top} \regpregq) \nonqv_i )  \label{eq:zetai}
\end{align}
Then it is straightforward to show that 
\begin{align*}
	\alpha_{i,j} &= \eta_i + \beta_j z_{i,j}^2 + 2 \zeta_{i,j} z_{i,j} \\
	\beta_j &= \gammadiff_{j,j} \\
	g_{i,j} &= \zeta_{i,j} + \beta_j z_{i,j}.
\end{align*}
Also define  $\regsquare^{\mathbf{a}, \mathbf{b}}(t)$ recursively, for any two vector-valued signals $\mathbf{a}(t)$, $\mathbf{b}(t)$, as
\begin{align}
	\regsquare^{\mathbf{a}, \mathbf{b}}(0) &= \0 \\ 
	\regsquare^{\mathbf{a}, \mathbf{b}}(t+1) &= \regsquare^{\mathbf{a}, \mathbf{b}}(t) + \mathbf{a}(t) \mathbf{b}^\transp(t) \label{eq:regsquare_update}.
\end{align}
Note that $\regsquare^{\mathbf{a}, \mathbf{b}}(t) = (\regsquare^{\mathbf{b}, \mathbf{a}}(t))^{\transp}$. It can be verified that all quantities \eqref{eq:gammadiff}--\eqref{eq:zetai}, and thus $\hat{z}_{i,j}$, can be coputed from $\regsquare^{\cdot, \cdot}(N)$. 

The full algorithm for updating the needed quantities, including the update of $\wbar{\linp}$ and $\regpregq$, is summarized in Algorithm \ref{alg:rec}. Note that the iterations of the outer for-loop can be executed in parallel. 
%Furthermore 
%\begin{align*}
%	\gammadiff &= \regq \regq^{\top} - \regq \regp^{\top} \regpregq - \regpregq^{\top} \regp \regq^{\top} + \regpregq^{\top} \regp \regp^{\top} \regpregq^{\top} \\
%	&= \regsquare^{\regqk, \regqk}(N) - \regsquare^{\regqk, \regpk}(N) \regpregq - \regpregq^{\top} \regsquare^{\regpk, \regqk}(N) + \regpregq^{\top} \regsquare^{\regpk, \regpk}(N) \regpregq.
%\end{align*}
%In the same way it can be verified that
%\begin{align*}
%	\kappa_i &= \regsquare_{i,i}^{\out, \out}(N) + \tilde{\linpv}_i^{\top} \regsquare^{\regpk, \regpk}(N) \tilde{\linpv}_i - 2 \tilde{\linpv}_i^{\top} [ \regsquare^{\regpk, \out}(N)]_i \\
%	\rho_i &= [\regsquare^{\regpk, \out}(N)]_i - \regsquare^{\regpk, \regqk}(N)\tilde{\linpv}_i  \\
%	&\quad- \regpregq [\regsquare^{\regpk, \out}(N)]_i + \regpregq \regsquare^{\regpk, \regpk}(N) \tilde{\linpv}_i \\
%	\eta_i &= \kappa_i - 2\rho_i^{\transp} \nonqv_i + \nonqv_i^{\transp} \gammadiff \nonqv_i \\
%	\zeta_i &= \rho_i - \gammadiff \nonqv_i.
%\end{align*}
%Hence, all the variables we need to find $\hat{z}_{i,j}$ can be computed using $\regsquare^{\cdot, \cdot}(N)$, and the latter matrices can be computed recursively. 
%In the above derivation we did not make any assumptions on $\weights_i$, so these weights can either be precomputed, or possibly computed recursively. We can see that for the case with $\wtilde{\para}_0 = \{ \wtilde{\linp}_0, \0, \wtilde{\covnoise}_0 \}$, the weights are given by
%\[
%	w_{i,j}^2 =\frac{1}{N} \|\wbar{\regqk}_j\|^2_2 = \ \frac{1}{N}\left( \regq \regq^{\top} \right)_{j,j} = \frac{1}{N}\regsquare^{\regqk, \regqk}_{j,j}(N).
%\]

\begin{algorithm}[h!]
	\caption{: \textbf{Recursive solution to \eqref{eq:concentrated}}} \label{alg:rec}
\begin{algorithmic}[1]
	\State Input: $\out(t)$, $\regpk(t)$, $\regqk(t)$, $\check{\linp}$ and $\check{\nonq}$
	\State Update $\gain(t)$, $\wbar{\linp}(t)$ and $\regpregq(t)$ according to \eqref{eq:RLS_linp}-\eqref{eq:RLS_gain}.
	\State Update $\regsquare^{\regpk, \regpk}(t)$, $\regsquare^{\regqk, \regqk}(t)$, $\regsquare^{\out, \out}(t)$,  $\regsquare^{\regpk, \regqk}(t)$, $\regsquare^{\regpk, \out}(t)$ and $\regsquare^{\regqk, \out}(t)$ according to \eqref{eq:regsquare_update} .
	\State $\gammadiff = \regsquare^{\regqk, \regqk} - \regsquare^{\regqk, \regpk} \regpregq - \regpregq^\transp \regsquare^{\regpk, \regqk} + \regpregq^{\transp} \regsquare^{\regpk, \regpk} \regpregq$. 
	\For{$i=1, \dots, n_y$}
		\State	$\kappa = \regsquare^{\out, \out}_{i,i} + \tilde{\linpv}_i^{\transp} \regsquare^{\regpk, \regpk} \tilde{\linpv}_i - 2 \tilde{\linpv}_i^{\transp} [\regsquare^{\regpk, \out}]_i  $.
		\State	$\rho = [\regsquare^{\regqk, \out}]_i - \regsquare^{\regqk, \regpk} \tilde{\linpv}_i - \regpregq^\transp [\regsquare^{\regpk, \regqk}]_i + \regpregq^\transp \regsquare^{\regpk, \regpk} \tilde{\linpv}_i $.
		\State  $\eta = \kappa - 2 \rho^{\transp} \check{\nonqv}_i + \check{\nonqv}_i^{\transp} \gammadiff \check{\nonqv}_i $.
		\State  $\zeta = \rho - \gammadiff \check{\nonqv}_i$
		\Repeat
		\For{ $j= 1, \dots, q$ }
			\State $\alpha = \eta + \gammadiff_{j,j} \check{z}_{i,j}^2 + 2 \zeta_j \check{z}_{i,j} $.
			\State $g = \zeta_j + \gammadiff_{j,j} \check{z}_{i,j}$.
			\State $\beta = \gammadiff_{j,j}$.
%			\State $w = \sqrt{\regsquare^{\regqk, \regqk}_{j,j}/t}$.
			\State $\hat{r} =\frac{|g|}{\beta} - \frac{w_{i,j}}{\beta\sqrt{\beta - w_{i,j}^2}}\sqrt{\alpha \beta - g^2} $. 
			\State $ \hat{z}_{i,j} = \left\{ \begin{array}{ll} \sign(g) \hat{r} & \text{if } \alpha w_{i,j}^2 < g^2 \\ 0 & \text{otherwise}\end{array} \right.$
			\State $\eta := \eta + \gammadiff_{j,j}(\check{z}_{i,j} - \hat{z}_{i,j}) + 2 (\check{z}_{i,j} - \hat{z}_{i,j}) \zeta_j$.
			\State $\zeta := \zeta + [\gammadiff]_j (\check{z}_{i,j} - \hat{z}_{i,j})$.
			\State $\check{z}_{i,j} := \hat{z}_{i,j}$.
		\EndFor
		\Until {number of iterations equals $L$.}
	 \EndFor
	 \State $\what{\nonq} = \check{\nonq}$.
	 \State $\what{\linp} = \wbar{\linp} - \what{\nonq} \regpregq$.
	 \State Output: $\what{\linp}, \what{\nonq}$.
\end{algorithmic}
\end{algorithm}

\bibliographystyle{plain}
\bibliography{refs_nonlinear}

\end{document}